\documentclass[lettersize,journal]{IEEEtran}
\usepackage{amsmath,amsfonts}
\usepackage{algorithmic}
\usepackage{amssymb}
\usepackage[dvipsnames,table]{xcolor}
\usepackage{algorithm}
\usepackage{listings}
\usepackage{array}
\usepackage[caption=false,font=normalsize,labelfont=sf,textfont=sf]{subfig}
\usepackage{textcomp}
\usepackage{stfloats}
\usepackage{url}
\usepackage{verbatim}
\usepackage{graphicx}
\usepackage{subcaption}
\usepackage{cite}
\usepackage{amsthm}
\usepackage{hyperref}
\usepackage{tabularray}
\usepackage{multirow}
\usepackage{caption}
\usepackage{color}
\usepackage[normalem]{ulem}
\usepackage{booktabs}
\usepackage{siunitx}
\usepackage{multirow}
\usepackage{xcolor}
\usepackage{colortbl}

\usepackage{xspace}
\usepackage{lipsum}
\newcommand{\latinphrase}[1]{\textit{#1}} 
\newcommand{\etal}{\latinphrase{et~al.}\xspace}
\newcommand{\ie}{\latinphrase{i.e.}\xspace}

\newtheorem{theorem}{Theorem}[section]

\begin{document}

\title{A Unified Non-Parametric and Interpretable Point Cloud Analysis via t-FCW Graph Representation}
\author{\IEEEauthorblockN{Anonymous Authors}}
\author{Haijian Lai, Bowen Liu, Man Xu,
Chan-Tong Lam,  \IEEEmembership{Senior Member, IEEE},
Jo\~ao Macedo,
Benjamin Ng,  \IEEEmembership{Senior Member, IEEE},
and Sio-Kei Im,  \IEEEmembership{Senior Member, IEEE}
\thanks{Haijian Lai; Bowen Liu; Man Xu; Chan-tong Lam and Benjamin K. Ng are with the Faculty of Applied Sciences, Macao Polytechnic University, Macao SAR, China  (e-mail: \{haijian.lai, bowen.liu, p2109382, ctlam, bng\}@mpu.edu.mo). Chan-tong Lam is the corresponding author.}
\thanks{Jo\~ao Macedo is with the University of Coimbra, CISUC/LASI, DEI. (e-mail: jmacedo@dei.uc.pt)}
\thanks{Sio-kei Im is with Macao Polytechnic University, Macao SAR, China. (e-mail: marcusim@mpu.edu.mo)}
\thanks{This work was supported in part by the Foundation for Science and Technology (FCT) under Grant UID/00326. This work  was supported in part by the Science and Technology Development Fund of Macao under Grant 0033/2023/RIA1}}

\markboth{Journal of \LaTeX\ Class Files,~Vol.~14, No.~8, August~2021}%
{Shell \MakeLowercase{\textit{et al.}}: A Sample Article Using IEEEtran.cls for IEEE Journals}

\maketitle
\begin{center}
\parbox{0.95\linewidth}{\footnotesize\centering
\textit{This is the author's accepted manuscript of an article accepted for publication in \textit{IEEE Transactions on Multimedia}, 2026. The final version of record will be available via IEEE Xplore.}\\[0.4em]
\textcopyright~2026 IEEE. Personal use of this material is permitted. Permission from IEEE must be obtained for all other uses, in any current or future media, including reprinting/republishing this material for advertising or promotional purposes, creating new collective works, for resale or redistribution to servers or lists, or reuse of any copyrighted component of this work in other works.
}
\end{center}
\vspace{0.5em}

\begin{abstract}
We introduce an empowered transposed Fully Connected Weighted (t-FCW) graph representation to embed point clouds into a metric space. While original t-FCW has shown promising results for point cloud classification, the reasons behind its effectiveness and its broader applicability remained unclear. In this work, we analyze the properties that make the empowered and original t-FCW effective and design a network that uses the empowered t-FCW exclusively as feature extractors. From an interpretability perspective, we build memory banks for classification, part segmentation, and semantic segmentation using the empowered t-FCW. Our analysis reveals that the empowered t-FCW inherits robustness from surface descriptors, provides interpretability through dimension-wise relations. These properties enable a highly efficient and interpretable network, which processes the ModelNet40 classification problem in approximately 7 seconds on an NVIDIA RTX A5000 GPU. Importantly, empowered t-FCW can function both as a lightweight standalone baseline and as a complementary plug-in to existing deep models.
Our code is available at \href{https://github.com/hawkinglai/et-FCW}{https://github.com/hawkinglai/et-FCW}.
\end{abstract}


\section{Introduction} \label{sec:intro}
\IEEEPARstart{P}{oint} cloud classification and segmentation are foundational tasks \cite{sarker2024comprehensive,zhang2023deep,rani2024advancements} in 3D computer vision, important for research and multimedia applications such as robot sensing, scene sensing, etc. With the improvements in artificial intelligence, most current solutions are using deep learning to process 3D point cloud data. Deep learning solutions for point cloud processing offer several advantages. 
First, researchers do not require an extensive mathematical background for solution construction. 
For example, some deep learning frameworks \cite{FeyLenssen2019,Chaton2020TorchPoints3DAM} provide tools and effective methods that simplify the implementation. Moreover, other 3D deep learning implementations \cite{pytorchpointnet++,qian2022pointnext} provide code extensions for the original work or reproduce algorithms using more efficient techniques. Second, researchers can focus on the properties of data instead of designing features manually because deep learning models can learn or extract features from the raw data. Finally, neural networks are used to achieve state-of-the-art performance with training due to their ability to learn patterns and representations from large-scale training data.

Despite the above advantages, deep learning suffers from various disadvantages, including high computational requirements, black-box properties, overfitting, and being data-hungry. Consequently, non-parametric models \cite{zhang2023starting,zhu2024no} have emerged as a compelling research direction offering interpretability, though typically with lower performance than state-of-the-art neural networks. 
It is noticed that non-parametric models play a dual role: First, they serve as lightweight standalone encoders when interpretability, rapid deployment, or minimal computational resources are prioritized. Second, they serve as complementary auxiliary modules that enhance deep learning models by providing geometrically-grounded features without introducing 
trainable parameters. This duality points out that non-parametric methods serve not as replacements for deep learning, but rather as flexible methods that can work independently or synergistically depending on application requirements.
The pioneering works of non-parametric networks (e.g., Point-NN \cite{zhang2023starting} and Seg-NN \cite{zhu2024no}) for point cloud processing demonstrate both roles effectively. Their non-parametric versions prioritize 
interpretability and traceable feature extraction. Their parametric counterparts (i.e., Point-PN and Seg-PN) leverage the complementary features provided by Point-NN and Seg-NN to achieve better performance, prioritizing model performance. However, existing non-parametric methods encode point clouds into high-dimensional features, presenting a scalability challenge shared with many other approaches: Point-NN and Seg-NN require substantial Graphics Processing Units (GPU) memory as the point volume increases. To address this limitation, we propose the t-FCW framework with a significantly more compact feature representation while maintaining interpretability and the dual-role flexibility of non-parametric approaches.

Lai~\etal~\cite{lai2025point} proposed a transposed Fully Connected Weighted (t-FCW) graph representation for point cloud classification using Topological Data Analysis (TDA) to solve the problem of increasing computational complexity as point cloud volume grows. They discovered that features obtained from the transposed point cloud matrix and constructed as a fully connected weighted graph that can capture effective geometric information despite their compact size. In other words, the size of the t-FCW graph representation remains constant while the point cloud volume changes. Although the issue of computational complexity as the point volume grows can be addressed in classification, their work overlooked the features of every point in a point cloud, which poses a challenge when extending t-FCW to segmentation tasks due to the loss of point information.

In this paper, we improve the design of the t-FCW graph representation and extend it to part segmentation and semantic segmentation. It consists of an empowered t-FCW hierarchical framework for embedding the point cloud surface descriptors (PCSDs). 
Meanwhile, to ensure interpretability in the prediction, we adopt a similarity-based prediction model that computes pairwise similarity scores between entries in the memory banks constructed from the training and test sets ~\cite{zhang2023starting,zhu2024no,lai2025point}, which demonstrate the state-of-the-art performance in non-parametric point cloud analysis.
Furthermore, we hypothesize and verify that the effectiveness of t-FCW largely stems from the properties of the underlying PCSDs. Through systematic design and evaluation, we show that t-FCW inherits these properties in order to provide an efficient and robust representation. In this way, the t-FCW framework translates point clouds into compact descriptor-based graphs while preserving interpretability for point cloud analysis.
It is important to emphasize that t-FCW is not proposed as a universal substitute for parametric point cloud networks. Our contribution is to introduce an efficient, compact, and interpretable non-parametric representation that complements existing methods. The empowered t-FCW provides scalability to large point clouds and interpretability, making it suitable both as a lightweight standalone baseline and as a plug-in module within broader pipelines.
Our main contributions are:
\begin{enumerate}
    \item We propose an efficient, compact, and interpretable non-parametric framework for point cloud analysis based on the empowered t-FCW graph representation.
    \item We analyze the representation generated by the proposed framework that is based on the empowered t-FCW, demonstrating that its properties are inherited from the Gram matrix and various PCSDs from prior research.
    \item We conduct experiments on classification, part segmentation, and semantic segmentation benchmarks, and compare against existing non-parametric approaches. Results highlight t-FCW’s efficiency and scalability to large point clouds, as well as its complementary role integrated into deep learning models.
\end{enumerate}
This paper is organized as follows. Section \ref{sec:relatedwks} reviews related work and motivates the proposed t-FCW framework. Section \ref{sec:approach} details the framework’s architecture, while Section \ref{sec:exp} describes the experimental setup and reports results on multiple benchmarks, and Section \ref{sec:conclu} concludes the paper. 

\section{Related Works \& Preliminaries}\label{sec:relatedwks}
\subsection{Point Cloud Surface Descriptors}
PCSDs are used to describe the shapes of a point cloud. Generally, point-based models like PointNet series \cite{qi2017pointnet,qi2017pointnet++,ma2022rethinking,qian2022pointnext} take the 3D Cartesian coordinates as input to the networks. It is insufficient to capture geometric information due to a lack of relationships or interactions between points in a point cloud, although deep neural networks can extract the representation from the raw data. Therefore, additional geometric information is needed for the Cartesian coordinates as the representation of point clouds. 
Qiu~\etal~\cite{qiu2022geometric} propose a geometric point cloud descriptor that represents explicit geometric relations (\ie~normal vector, edge, and length information) in low-level space for better high-level representation while training neural networks. Ran~\etal~\cite{ran2022surface} focus on the local geometry in point clouds that are described as triangular and umbrella orientation, inspired by the idea that a local curve can be represented by derivatives in the Taylor Series.
While the above PCSDs have achieved notable success in extracting discriminative features from point clouds, they generally lack robustness against perturbations like rotation. Rotation-invariant PCSDs should also be focused on for more real-world applications in the future, where objects may appear under arbitrary orientations. Zhang~\etal~\cite{zhang2024risurconv} propose a well-preserved rotation-invariant surface descriptor, which constructs a triangle surface for each point and its neighbors in a point cloud. The surface descriptors are used to capture richer geometric context from raw point clouds by describing local surface properties, enabling more robust, invariant, and effective feature extraction in point cloud analysis.
In contrast, other solutions, which transform a point cloud into another representation, enable existing techniques to process the point cloud indirectly, such as voxel \cite{voxnet2015ma}, graph \cite{wang2019dynamic}, multi-view images \cite{Su_2015_ICCV}, and so on. 
However, these transformations \cite{voxnet2015ma,Su_2015_ICCV,ran2022surface} are criticized for the following reasons: complex processing, high computational cost, and information loss.

\subsection{Non-parametric Models}
Non-parametric techniques have been developed to address the issues wherein parametric techniques typically assume known values including priors and relationships in the data. In this scenario, the assumed priors \cite{goyal2017nonparametric} often lose the rich semantics representation, which means that they cannot capture the intrinsic patterns and relationships of the data. Similarly, classical stochastic optimization methods, for example, Lipschitz functions \cite{malherbe2017global}, require extensive trial and error to search optimal parameters. To tackle the above issues, non-parametric techniques make fewer or no assumptions compared to parametric techniques, which do not rely on given priors or parameter collections.
Different from the above definition, Zhang~\etal~\cite{zhang2023starting} and Zhu~\etal~\cite{zhu2024no} presented Point-NN and Seg-NN, respectively without trainable parameters for point cloud analysis using trigonometric functions. They are inspired by Tancik~\etal~\cite{tancik2020fourier}, who discussed that high-frequency content can be effectively learned by Multi-Layer Perceptrons (MLPs) after sinusoidal mapping in low-dimensional data. Point-NN and Seg-NN introduced an embedding algorithm that refers to positional encoding \cite{vaswani2017attention}, which is used to encode the point cloud as a type of `signal'. Point-NN can capture high-frequency 3D structures, including edges and other geometric details. In contrast, Seg-NN focuses on analyzing point clouds in scenes that contain much high-frequency noise. Therefore, it is necessary to filter out the frequencies to refine the scene representation. 
Lai~\etal~\cite{lai2025point} observed that computational complexity increases with point cloud volume growth and proposed Point-FCW, a TDA framework that shows promising performance but lacks an explanation of why the t-FCW graph representation is effective.
In contrast, Gou~\etal~\cite{gou20253d} observed the limited interpretability of the above models that aggregate local point cloud features, and proposed the ICP-Classifier which is a neural-network-free and training-free pipeline that leverages Iterative Closest Point registration and $k$-NN for transparent and robust shape classification.

\subsection{Interpretable Models}
Interpretable models are defined as those whose internal mechanisms, such as the mapping from inputs to outputs, are sufficiently transparent to allow human to comprehend their decision-making processes. 
This interpretability facilitates the identification and mitigation of issues by enabling systematic inspection and validation of the model's logic.
Esser~\etal~\cite{esser2024non} introduces a kernel-based framework that enhances interpretability in representation learning by grounding its mechanisms in transparent mathematical structures. Unlike the black-box deep neural networks, this approach leverages kernels to define similarity between data points by contrasting learning, enabling users to inspect kernel matrices and understand how representations are formed. Feng~\etal~\cite{feng2024interpretable3d} present a prototype-based classification framework that makes the decision-making process of 3D deep learning models transparent and comprehensible to users. 
The model classifies point clouds by comparing them to interpretable prototypes for predictions in clear, inspectable examples, which allows users to trace how new inputs align with known patterns via visual and statistical similarity scores.

\begin{figure*}
    \centering
    \includegraphics[width=0.95\linewidth]{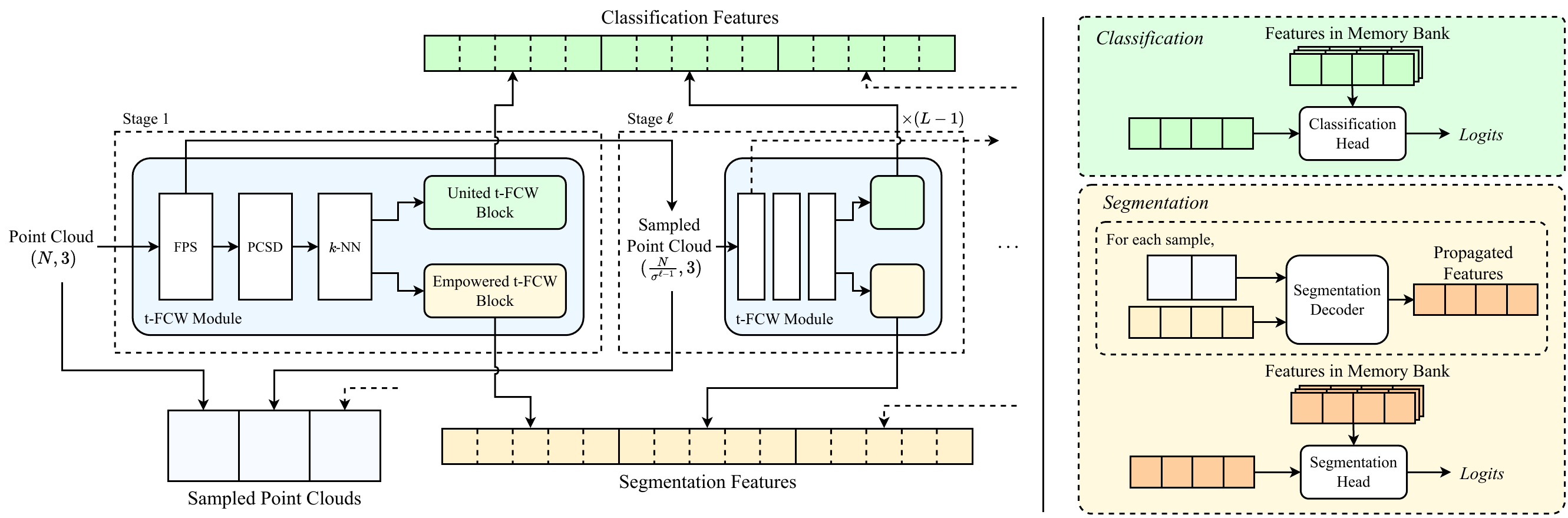}
    \caption{Feature extraction of the t-FCW networks.}
    \label{fig:fetfcwn}
\end{figure*}

\subsection{Preliminaries}
Processing point cloud with TDA \cite{tauzin2021giotto} is able to be considered as a specific case of fully-connected-weighted (FCW) graphs.
The procedure in the Vietoris–Rips filtration on a point cloud is often described as growing spheres with increasing radius around each point. First, in the complex-building step, given a metric space $(X, d)$ and a threshold $\epsilon \geq 0$, the Vietoris-Rips complex denoted by $VR_{\epsilon}(X)$ and it contains all subsets as complex if all subsets' pairwise distance less than or equal to the threshold $\epsilon$, explicitly:
\begin{equation}
    VR_{\epsilon}(X) = \Big\{\{x_{0}, ..., x_{n}\} \in X \; \vert \; \text{dist}(x_{i},x_{j} \leq \epsilon) \; \forall i,j \Big\}.
\end{equation}
Finally, in the filtration step, given a simplicial complex $X$, a filtration is a nested subcomplex of $X$ obtained through a monotonic function. Vietoris-Rips filtration arranges $VR_{\epsilon}(X)$ as a finite increasing sequence:
\begin{equation}
    \emptyset = VR_{\epsilon_0}(X) \subseteq VR_{\epsilon_1}(X) \subseteq \dots \subseteq VR_{\epsilon}(X) \subseteq X,
\end{equation}
where $0 < \epsilon_0 < \epsilon_1 < \dots, < \epsilon$.

According to the analysis in Point-FCW \cite{lai2025point}, although its graph construction strategy appears conceptually well-motivated, the method still achieves relatively low overall accuracy on point cloud classification tasks. Moreover, its performance is not competitive with state-of-the-art models and, in some cases, even falls short of baseline methods. 
Motivated by this, the t-FCW graph representation was introduced by transposing the point cloud matrix and constructing fully connected weighted graphs along the feature dimension rather than across points. This yields a compact and computationally efficient representation that has shown promise for classification. However, several limitations remain: (1) Existing t-FCW formulations are restricted to global classification and do not generalize to point-wise segmentation tasks, since point-level information is lost during feature aggregation. (2) While effective in practice, the reasons behind the strong performance of such a compact representation have not been fully analyzed.

\section{Proposed Framework}\label{sec:approach}

\subsection{Overview}\label{sec:framework}

The proposed framework is composed of two primary stages: feature extraction, which encodes geometric and structural information from the point cloud, and optimized prediction, which performs classification or segmentation efficiently.

\subsubsection{Feature Extraction}
We design a non-parametric and interpretable framework that contains various networks for different tasks, as shown in Fig. \ref{fig:fetfcwn}, inspired by recent works such as Point-MLP~\cite{ma2022rethinking}, Point-NN~\cite{zhang2023starting}, and Seg-NN~\cite{zhu2024no}. 
The network processes input point clouds through a sequence of sampling stages, each of which halves the point cloud volume. A module formats PCSDs from the point cloud, and the sampled PCSDs then obtain $K$ nearest neighbors from the unsampled PCSDs.
Specifically, the sampling procedure follows the workflow described in Section \ref{sec:sampling}, utilizing FPS (Furthest Point Sampling \cite{pytorchpointnet++}) followed by $k$-NN grouping to establish local neighborhoods. For each sampled PCSD in classification, the network constructs two complementary feature representations by using the United t-FCW Block as shown in Fig. \ref{fig:tfcw_block}: (1) a global t-FCW obtained by applying Algorithm \ref{alg:tfcw}, and (2) a local empowered t-FCW as shown in Algorithm \ref{alg:etfcw}, that preserves point-wise geometric relationships by applying Algorithm \ref{alg:etfcw}. The global t-FCW and local empowered t-FCW are aggregated using hybrid pooling that combines max and average pooling.
From the perspective of segmentation, the global t-FCW~\cite{lai2025point} is limited by its lack of point-wise spatial information. To address this, it is replaced with a neighborhood-based PCSD representation, where $K$ nearest neighbors are pooled to align the feature dimension $N$. Hence, as shown in Fig. \ref{fig:tfcw_block}, the framework utilizes the Empowered t-FCW Block for the segmentation network, which contains a pooled PCSD aligned with the $N$ dimension and a local empowered t-FCW embedding. These two features are concatenated to form a feature that preserves point-wise spatial information.
\begin{figure}[ht]
    \centering
    \includegraphics[width=0.85\linewidth]{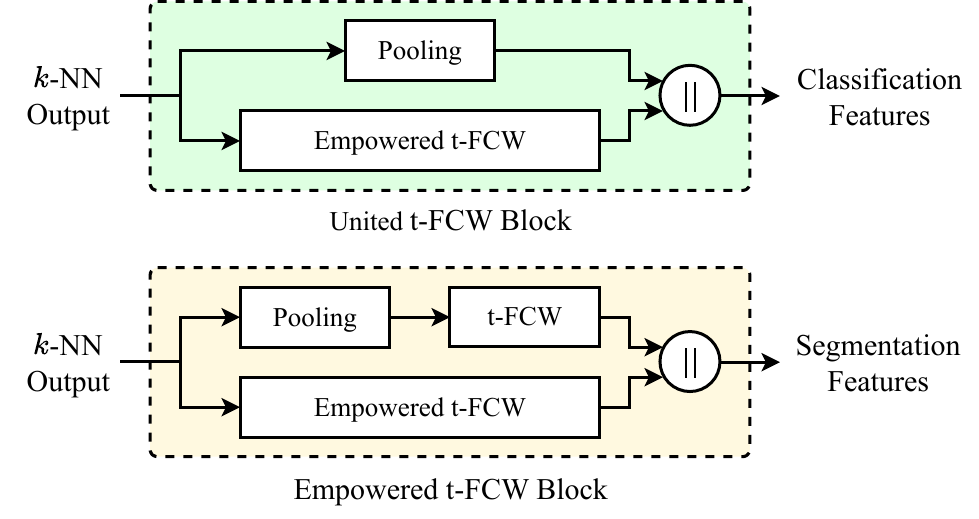}
    \caption{Structures of united and empowered t-FCW blocks}
    \label{fig:tfcw_block}
\end{figure}

Finally, the features at every stage on both networks enable the capture of both global geometric information and fine-grained local geometric details. 
For the classification networks, the features are appended into a list along the batch dimension to form a memory bank, denoted by $M$. For segmentation, the features are passed through a hierarchical decoder that performs feature propagation via inverse-distance weighted $k$-NN interpolation. At each stage, interpolated features are concatenated with skip-connected encoder features to recover dense per-point representations. The recovered per-point features are then collected to form the memory bank $M$.

\subsubsection{Prediction}
For the prediction, we adopt the prediction modules from Point-NN \cite{zhang2023starting} and Seg-NN \cite{zhu2024no}, as their classifiers currently represent interpretability and the state-of-the-art performance in non-parametric point cloud analysis. In all experiments, predictions are made using the similarity-based memory banks. 
Specifically, we construct a memory bank from the training set, $M_{\text{train}}$, and we use the same framework to obtain features from every test sample, denoted by $M_{\text{test}}$. 
Following the standard formulation of memory-based prediction \cite{zhang2021tip, zhang2023starting, zhu2024no} and its intuitive reformulation \cite{lai2025point}, the prediction logits are computed as
\begin{equation}
    \begin{aligned}\label{eq:tfeq}
        Logits &= a(S)L_{\text{one-hot}} \\
        &= a(M_{test}M_{train}^T)L_{\text{one-hot}},
    \end{aligned}
\end{equation}
where $a(S)=\exp(-\gamma (1 - S))$ is an activation function that emphasizes similarity scores $S$, and $\gamma$ is a tunable parameter determined through empirical validations \cite{zhang2023starting}. Meanwhile, the $L_{\text{one-hot}}$ denotes the one-hot encoded labels of the training set. Lai~\etal\cite{lai2025point} visualize the prediction modules, where the diagonal elements of $M_{test}M_{train}^T$ capture self-similarity scores of correct samples exhibiting maximal similarity, providing an interpretable measure of retrieval accuracy.

\begin{figure*}[ht]
    \centering
    \includegraphics[width=0.8\linewidth]{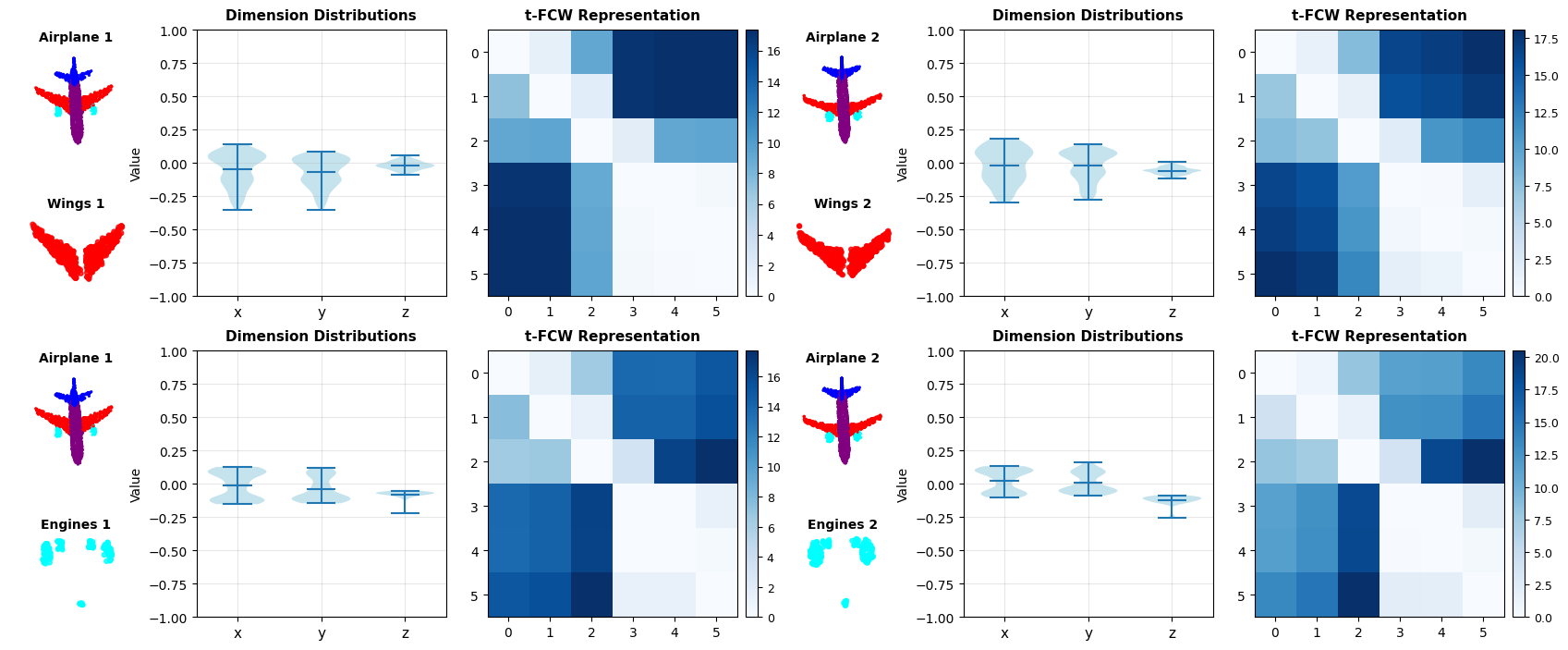}
    \caption{Wings and engines of airplanes with their dimension distributions and t-FCW representations, where the cosine similarity of wings 1\&2 is 99.0\%, engines 1\&2 is 99.4\%, and between engine-wing pairs 1\&2 is 93.0\% and 88.1\%.}
    \label{fig:objectsfeaturedis}
\end{figure*}

\subsection{t-FCW Graph Representation}
As illustrated in Section \ref{sec:framework} and Fig. \ref{fig:fetfcwn}, the t-FCW serves as the encoder within our proposed framework. In this section, we analyze the interpretability of t-FCW to explain its effectiveness in extracting discriminative yet interpretable features for classification or segmentation. We then discuss the theoretical properties of t-FCW, which is necessary to integrate t-FCW into the proposed framework or other existing methods.
\subsubsection{Interpretability of t-FCW}\label{sec:interpretability}
The t-FCW graph can be interpreted as a special dimension-dimension Gram matrix in a vectorized way. Specifically, it measures the pairwise correlation between feature dimensions in a point cloud rather than between points, thereby providing insights into how different geometric or surface descriptors interact.
Given a transposed point cloud $X = [v_1 \; v_2 \; \cdots \; v_d] \in \mathbb{R}^{d \times n}$, where $v_i \in \mathbb{R}^n$ represents the $i$-th dimension (e.g., $x, y, z$) across all $n$ points. The pairwise distance between two dimensions $i$ and $j$, denoted as $\text{t-FCW}(i,j)$, is defined as
\begin{equation}
    \begin{aligned}\label{eq:tfcwpd}
        \text{t-FCW($i,j$)} &= \sqrt{||v_i-v_j||^2} \\
        &= \sqrt{||v_i||^2 + ||v_j||^2 - 2 \langle v_i, v_j \rangle}.
    \end{aligned}
\end{equation}

Then, we can calculate the Gram matrix $G=XX^T  \in \mathbb{R}^{d \times d}$ with $G_{i,j}=\langle v_i,v_j \rangle$ and its diagonal elements vector $\text{diag}(G)$. Given that $1_n \in \mathbb{R}^n$ is a vector of ones, such that
\begin{equation}\label{eq:grammatrix1n}
    \text{diag(G)} = \begin{pmatrix}
        \langle v_1,v_1\rangle\\
        \langle v_2,v_2\rangle\\
        \vdots \\
        \langle v_d,v_d\rangle\\    
    \end{pmatrix}, 
    1_n = \begin{pmatrix}
        1\\
        1\\
        \vdots \\
        1\\    
    \end{pmatrix}.
\end{equation}

Based on the Equation (\ref{eq:tfcwpd}) and Equation (\ref{eq:grammatrix1n}), the t-FCW for all dimension pairs can be written as
\begin{equation}
    \text{t-FCW} = \sqrt{\text{diag}(G)1^{T}_{n}+1_{n}\text{diag}(G)^{T}-2G}.
    \label{eq:tfcwG}
\end{equation}

For the Gram matrix $G = XX^T \in \mathbb{R}^{d \times d}$, it characterizes the self-correlation across feature dimensions, making it a meaningful representation for distinguishing between different semantic objects. This concept has been supported in prior work: Seg-PN (parametric version of Seg-NN \cite{zhu2024no}) demonstrates that such dimension-wise correlations can encode object-level semantics effectively. Similarly, Point-FCW \cite{lai2025point} also confirms that feature channels from PCSD are not independent and exhibit dimensional correlations.
In addition, other studies have leveraged the Gram matrix as a self-correlation representation in similar contexts \cite{jing2019neural, lang2022learning}. 

The interpretability of t-FCW follows from its inter-dimensional correlations encoded through the Gram matrix, and each vector in the Gram matrix corresponds to the distribution of all points along a geometric descriptor dimension (e.g., xyz, normal vectors, etc). Semantic parts with consistent geometric structure exhibit stable 1D distributions across instances: airplane wings have a wingspan with $x$-spread, a chord length with $y$-spread, and Airfoil thickness judged by $z$-spread, while engines show compact, roughly uniform $xyz$-spread because of their cylindrical geometry. These characteristic distributions arise from stable inter-dimensional correlations $G_{i,j} = \langle v_i, v_j\rangle$, which t-FCW transforms into interpretable pairwise distances.

Fig. \ref{fig:objectsfeaturedis} validates this intuition across four airplane instances. Wings and engines exhibit consistent dimension distribution patterns across instances, producing similar t-FCW representations with 99.0\% and 99.4\% cosine similarity, respectively. Crucially, the comparisons of wing–engine pairs yield much lower similarities (93.0\% and 88.1\%), even though these parts are spatially adjacent on the object. This indicates that t-FCW encodes semantic geometry rather than spatial proximity, which incorrectly treats nearby but semantically unrelated parts as similar.
The resulting t-FCW thus provides a geometric representation that measures how descriptor dimension varies across a part, so it has clear part-level semantic meaning, bridging mathematical formulation with interpretability.

\subsubsection{Properties of t-FCW}
As shown in Equation (\ref{eq:tfcwG}), the t-FCW can be derived from the Gram matrix, so it inherits the properties from the Gram matrix. For the generalization in point cloud analysis, we need to consider two basic properties:
\begin{itemize}
    \item \textbf{Permutation invariance}: the order of a point cloud is random and should not affect the extracted features. 
    \item \textbf{Rotation invariance}: the shape remains consistent under rotations, hence the orientation of a point cloud in 3D space cannot affect the extracted features.
\end{itemize}
However, it is noted that these properties exist strictly within the scope of the t-FCW construction and may not obtain the same properties if it is extended to other representations beyond this formulation. The following two theorems show that while t-FCW naturally achieves permutation invariance, it does not preserve rotation invariance.

\begin{theorem}[Permutation Invariance of Gram matrix]
\label{pythagorean}
Given a transposed point cloud \(X \in \mathbb{R}^{d \times n}\), applying a permutation $P$ to the input point cloud does not alter the Gram matrix output. Mathematically, for a point cloud X and a permutation matrix $P$, we aim to show:

\begin{equation}
    G(XP) = G(X).
\end{equation}
\end{theorem}

\begin{proof}
For a point cloud \(X \in \mathbb{R}^{d \times n}\) and its corresponding Gram matrix, a permutation matrix $P$ is applied to $X$:
\begin{equation}
    G(XP) = (XP)(XP)^T = XPP^TX^T = XX^T = G(X).
\end{equation}
Therefore, the Gram matrix is permutation invariant.
\end{proof}

\begin{theorem}[Rotation Variance of Gram matrix]\label{theorem:rvgm}
Let \(X_c \in \mathbb{R}^{3 \times n}\) be a transposed point cloud in a 3D Cartesian coordinate system, excluding cases of 3D objects with rotational symmetry (e.g., sphere). For a generic rotation \(R \in SO(3)\) with \(R \neq I\), where $SO(3)$ is the special orthogonal group for rotations in 3D space that preserve orientation, applying \(R\) to \(X_c\) generally alters the Gram matrix. Mathematically, for the point cloud $X_c$ and a rotation matrix $R$, we aim to show:

\begin{equation}
    G(RX_c) \neq G(X_c).
\end{equation}
\end{theorem}

\begin{proof}
For a generic point cloud \(X_c \in \mathbb{R}^{3 \times n}\) and its corresponding Gram matrix, a generic rotation $R$ is applied to $X_c$:
\begin{equation}
    G(RX_c) = (RX_c)(RX_c)^T = RX_cX_c^TR^T = RG(X_c)R^T.
\end{equation}
As $RG(X_c)R^T = G(X_c) $ generally does not hold for $X_c$ and $R$, we conclude that, for a generic rotation $R \neq I$, the Gram matrix typically changes: $G(RX_c) \neq G(X_c)$.
\end{proof}

As a result, the same object under different orientations causes different t-FCW representations, thereby affecting recognition or classification. 
This necessarily motivates us to introduce a PCSD that can decouple orientation-dependent components from intrinsic shape information for rotation invariance in point cloud processing.

\subsection{Components}\label{sec:methods}
\subsubsection{PCSDs Formatting}
Given a point cloud $P \in \mathbb{R}^{N\times 3}$ that exists in a 3D space with $N$ points, where the collection of points has the relation in $P = \{p_1, p_2, \dots, p_N\}$ to stand for a shape. 
For any $p_i \in P$, it requires at least a coordinate system, e.g., it can be represented as $(x_i, y_i, z_i)$ by using the Cartesian coordinate system. Additionally, the point cloud contains not only the coordinate system but also other supplementary information (color, normal vector, etc.). 
Moreover, Qiu~\etal~\cite{qiu2022geometric} take additional geometric information to provide more expressive information about the explicit geometric information in low-level space for making full use of the triangle features. Given a point $p_{i0} \in P$ and two points ($x_{i1}$, $x_{i2}$) which are the nearest neighbors of $p_{i0}$, their geometric point descriptor $\widetilde{p_{i0}}$ (denoted as \textit{GeoPCSD}) is:

\begin{equation}
    \begin{aligned}\label{eq:gbp_pcsd}
    \overrightarrow{V_{01}} &= x_{i1} - p_{i0} \in \mathbb{R}^{3},\\
    \overrightarrow{V_{02}} &= x_{i2} - p_{i0} \in \mathbb{R}^{3},\\
    \widetilde{p_{i0}} &= \left[p_{i0}, \overrightarrow{V_{01}} \times \overrightarrow{V_{02}}, \overrightarrow{V_{01}}, \overrightarrow{V_{02}}, 
    \Vert \overrightarrow{V_{01}} \Vert, \Vert \overrightarrow{V_{02}} \Vert \right] \in \mathbb{R}^{14} ,
    \end{aligned}
\end{equation}
where $\widetilde{p_{i0}}$ denotes the extended representation of $p_i$, which is 3D coordinate information with normal vector $\overrightarrow{V_{01}} \times \overrightarrow{V_{02}}$ and estimated edges (e.g., $\overrightarrow{V_{01}}$). Vector-based information is invariant to translation since it relies on relative differences between points, so the additional geometric information can provide robustness to real-world data challenges \cite{qiu2022geometric}.

From Theorem~\ref{theorem:rvgm}, the Gram matrix is not invariant under generic rotations. Therefore, it is necessary to identify a PCSD that remains consistent regardless of the rotation applied to the point cloud for robustness. 
Zhang~\etal~\cite{zhang2024risurconv} propose \textit{RISP}, which is a rotation-invariant surface descriptor by constructing triangles. RISP carries explicit geometric meaning, which is crucial for explaining the downstream task, whereas the learned rotation-invariant features lack this clarity.
It formats two local triangle surfaces around each neighbor of a reference point using its adjacent neighbors.
First, they aim to construct a relationship between the two triangle surfaces by angles. Given a point $p_{i} \in P$ and three neighboring points ($x_{i-1}$, $x_{i}$, and $x_{i+1}$) from $p_i$, we have five angles $\phi$ that characterize the two triangles and their surface relationship with respect to the edge $x_ip_i$ in Euclidean space:
\begin{equation}
    \begin{aligned}
        &\phi_1= \angle (\overrightarrow{x_{i-1}p_i}, \overrightarrow{x_ip_i}), \phi_2= \angle (\overrightarrow{x_{i+1}p_i}, \overrightarrow{x_ip_i}), \\
        &\phi_3= \angle (\overrightarrow{x_{i-1}x_i}, \overrightarrow{x_{i-1}p_i}), \phi_4= \angle (\overrightarrow{x_{i+1}p_i}, \overrightarrow{x_{i+1}x_i}), \\
        &\phi_5= \angle (\overrightarrow{x_{i+1}p_i} \times \overrightarrow{x_ip_i}, \overrightarrow{x_{i-1}p_i} \times \overrightarrow{x_ip_i}).
    \end{aligned}
\end{equation}

Secondly, they aim to extend this representation of the two triangle surfaces using normal vectors, which capture how the surfaces bend away from the tangent space:
Given the normal vectors $n_{p_i}$, $n_{i}$, $n_{i-1}$, and $n_{i+1}$ of the above points, respectively, we have:
\begin{equation}
    \begin{aligned}
        \alpha_1= \angle (\overrightarrow{n_{p_i}}, \overrightarrow{x_ip_i}), &\alpha_2= \angle (\overrightarrow{n_{p_i}}, \overrightarrow{x_{i-1}p_i}), \\
        \beta_1= \angle (\overrightarrow{n_i}, \overrightarrow{x_ip_i}), &\beta_2= \angle (\overrightarrow{n_i}, \overrightarrow{x_{i-1}x_i}), \\
        \theta_1= \angle (\overrightarrow{n_{i-1}}, \overrightarrow{x_{i-1}p_i}), &\theta_2= \angle (\overrightarrow{n_{i-1}}, \overrightarrow{x_{i-1}x_i}), \\
        \gamma_1= \angle (\overrightarrow{n_{i+1}}, \overrightarrow{x_{i+1}x_i}), &\gamma_2= \angle (\overrightarrow{n_{i+1}}, \overrightarrow{x_{i+1}p_i}).
    \end{aligned}
\end{equation}

Finally, the rotation-invariant PCSD (i.e., \textit{RISP}) can be expressed, explicitly:
\begin{equation}\label{eq:ripcsd}
\begin{aligned}
    \widetilde{p_{i}} = [ &L, \phi_1, \phi_2, \phi_3, \phi_4, \phi_5, \\
    &\alpha_1, \alpha_2, \beta_1, \beta_2, \theta_1, \theta_2, \gamma_1, \gamma_2 ] \in \mathbb{R}^{14}.
\end{aligned}
\end{equation}
\noindent where $L$ is the distance from reference $p_i$ to its neighbors $x_i$ and we know \textit{RISP} is an orientation-agnostic PCSD. We have,
\begin{theorem}[Rotation Invariance of t-FCW w/ \textit{RISP}]\label{theorem:ritfcwwrisp}
Given a point cloud $X_c \in \mathbb{R}^{3 \times n}$ in Cartesian coordinates, and let $R \in SO(3)$ be the rotation. Define a \textit{RISP} $S_{RI}: \mathbb{R}^{3 \times n} \rightarrow \mathbb{R}^{14 \times n}$, as introduced in Equation (\ref{eq:ripcsd}), such that
\begin{equation}
S_{RI}(RX_c) = S_{RI}(X_c).
\end{equation}
Then, the Gram matrix computed from the transformed representation is rotation-invariant:
\begin{equation}
    G(S_{RI}(RX_c)) = G(S_{RI}(X_c)).
\end{equation}
\end{theorem}

In this work, we employ three types of PCSDs: (1) the raw Cartesian coordinates (\textit{xyz}), which provide the most basic geometric representation; (2) the \textit{GeoPCSD} \cite{qiu2022geometric}, which enriches each point with additional geometric relations such as edge lengths and normal vectors, thereby enhancing robustness to noise; and (3) the \textit{RISP} \cite{zhang2024risurconv}, which encodes local triangular surface structures to ensure invariance under arbitrary 3D rotations. Each of these descriptors can be converted into a t-FCW graph, allowing the framework to directly inherit their respective properties (e.g., robustness from GeoPCSD and rotation invariance from RISP).

\subsubsection{Sampling}\label{sec:sampling}
Processing the full-resolution point cloud is computationally consuming due to its large data volume and memory requirements. Hence, PCSDs are typically sampled to reduce computational complexity, enabling efficient analysis without losing too much essential geometric information as possible. Following the design principles observed in Point-NN~\cite{zhang2023starting}, Seg-NN~\cite{zhu2024no}, and Point-FCW~\cite{lai2025point}, which assume that local neighborhoods provide reliable geometric information, we adopt similar strategies for sampling.

Given a point cloud $P$, we perform a sampling procedure consisting of FPS with 50\% sampling ratio, the computation of PCSDs denoted by $\mathcal{S}$, and $K$-NN ($K$-nearest neighbors) with $K$ neighbors.
Conceptually, the whole processing workflow is:
\begin{equation}
    \begin{aligned}\label{eq:fpsknn}
        P_{s} = FPS(P)&, P^{s}\in \mathbb{R}^{\frac{N}{2} \times 3},\\
        P^{\prime} = \mathcal{S}(P)&, P^{\prime} \in \mathbb{R}^{N \times C},  \\
        P_s^{\prime} = \mathcal{S}(P_{s})&, P_s^{\prime}\in \mathbb{R}^{\frac{N}{2} \times C}, \\
        P^{\prime\prime}_s = \text{$K$-NN}(P_s^{\prime}, P^{\prime})&, P^{\prime\prime}_s\in \mathbb{R}^{\frac{N}{2} \times K \times C}.\\
    \end{aligned}
\end{equation}

\subsubsection{Translating}

Lai~\etal~\cite{lai2025point} introduce the design of the t-FCW representation, which aims to translate an entire point cloud into a compact size. The completed implementation is provided in the Algorithm \ref{alg:tfcw}, where alpha \cite{tauzin2021giotto} denotes a scaling factor.
It presents a consistent size representation for point cloud processing, unlike other non-parametric algorithms \cite{zhang2023starting,zhu2024no} whose efficiency is strongly impacted as the dimension increases. However, it only focuses on global content, which means it cannot construct a memory bank for each point. As a result, it cannot do other tasks like part segmentation and semantic segmentation.

\begin{algorithm}[htbp]
\caption{\small Pytorch-style pseudocode of \textbf{t-FCW}}
\label{alg:tfcw}
\definecolor{codeblue}{rgb}{0.25,0.5,0.7}
\lstset{
  backgroundcolor=\color{white},
  basicstyle=\fontsize{7.6pt}{7.6pt}\ttfamily\selectfont,
  columns=fullflexible,
  breaklines=true,
  captionpos=b,
  commentstyle=\fontsize{10pt}{10pt}\color{codeblue},
  keywordstyle=\fontsize{10pt}{10pt},
}
\begin{lstlisting}[language=python]
# B:batch size;N,D,G,K:number of points,dimensions,grouping points,k-neighbors
b, n, k, d = surface.shape
surface = surface.permute(0, 3, 1, 2)
surface = pooling(surface) # [B, D, N]
x = pairwise_distance(surface, metric) # [B, D, D]
x[:, :-1, 1:] *= alpha
return x
    
\end{lstlisting}
\end{algorithm}

To address this limitation, we revisit the design of t-FCW to enable the construction of a t-FCW for each individual point, denoted as an empowered t-FCW. Rather than aggregating local information through pooling, which may discard useful geometric structure, we preserve the full local neighborhoods. This assumption is widely adopted in various existing methods \cite{zhang2023starting,zhu2024no,lai2025point} that treat local neighborhoods as reliable geometric information. As shown in Algorithm \ref{alg:etfcw}, we permute the PCSD matrix, which can be considered a collection of points with $K$ neighbors. It enables the translation of each point's neighbors as the t-FCW for maintaining all points.

\begin{algorithm}[ht]
\caption{\small Pytorch-style pseudocode of \textbf{empowered t-FCW}}
\label{alg:etfcw}
\definecolor{codeblue}{rgb}{0.25,0.5,0.7}
\lstset{
  backgroundcolor=\color{white},
  basicstyle=\fontsize{7.6pt}{7.6pt}\ttfamily\selectfont,
  columns=fullflexible,
  breaklines=true,
  captionpos=b,
  commentstyle=\fontsize{10pt}{10pt}\color{codeblue},
  keywordstyle=\fontsize{10pt}{10pt},
}
\begin{lstlisting}[language=python]
# B:batch size;N,D,G,K:number of poin
ts,dimensions,grouping points,k-neighbors
b, n, k, d = surface.shape
surface = surface.permute(0, 1, 3, 2)
surface = surface.reshape(-1, d, k) # [B*N, D, K]
std_dev = torch.std(surface, dim=-2, keepdim=True)+1e-5
surface /= std_dev
x = pairwise_distance(surface, metric) # [B*N, D, D]
x[:, :-1, 1:] *= alpha
return x.reshape(b, n, d, d)
    
\end{lstlisting}
\end{algorithm}

The comparison illustration for t-FCW and the empowered t-FCW is presented in Fig. \ref{fig:intro}.
It illustrates the difference between the original t-FCW and the proposed empowered t-FCW. The original t-FCW, where a single global pairwise distance matrix is derived from the entire point cloud, captures geometric structure but lacks point-level detail. In contrast, the empowered t-FCW at each point computes a local pairwise distance matrix based on its $K$ nearest neighbors.
This point-wise representation preserves local geometric relationships and enables carrying out tasks such as part segmentation and semantic segmentation.

\begin{figure}[ht]
\centering
\includegraphics[width=0.85\linewidth,trim=80 0 80 87, clip]{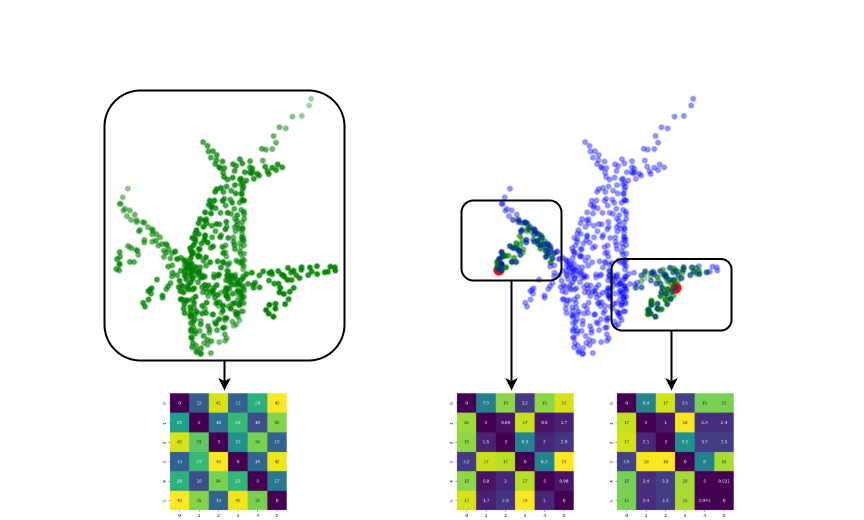}
\caption{Highlight the features of t-FCW (left) and empowered t-FCW (right). We use the $K$ nearest neighbors to extract similar semantic information, so that these engines of the airplane share a similar pattern.}
\label{fig:intro}
\vspace{-5pt}
\end{figure}

\vspace{-5pt}
\section{Experiments} \label{sec:exp}
As stated in Section \ref{sec:methods}, we evaluate t-FCW using three PCSDs: \textit{xyz} with $K$ nearest neighbor vectors $\in \mathbb{R}^{6}$; \textit{GeoPCSD} $\in \mathbb{R}^{14}$ \cite{qiu2022geometric} comprising coordinates, normal vectors, and edge information; and \textit{RISP} $\in \mathbb{R}^{14}$ \cite{zhang2024risurconv} comprising edge length and rotation-invariant angular features from local triangle surfaces. For benchmarking, the t-FCW framework is evaluated across three tasks: classification, part segmentation, and semantic segmentation. For generalization, we consider three aspects: robustness under rotation, data shuffling, and 3D adversarial attacks. We also conduct some ablation studies to discuss the impact of the diagonal elements for t-FCW, the impact of different modules in the t-FCW, and the impact of the selection of $K$ nearest neighbors for t-FCW. Finally, we conduct a simulation experiment for processing a growing-up point cloud using various networks.

\subsection{Experimental Settings}\label{sec:expsetting}
We report performance on four widely used benchmark datasets: ModelNet40 \cite{wu20153d} , ScanObjectNN \cite{uy2019revisiting}, ShapeNet-Part \cite{yi2016scalable}, and S3DIS \cite{armeni20163d}. We adopt the established data split principles for all datasets. Specifically, the predefined splits principles are provided by ModelNet40, ScanObjectNN, and ShapeNet-Part, while S3DIS is preprocessed according to the principles described in \cite{zhao2021few,zhu2024no}.
The t-FCW framework is compatible with various types of PCSDs. We consider the use of \textit{xyz} with $K$ nearest neighbors as the default PCSD, and it is denoted as \textbf{t-FCW} only. When \textit{GeoPCSD} or \textit{RISP} is used as the PCSD, we refer to the corresponding models as \textbf{t-FCW w/ \textit{GeoPCSD}} and \textbf{t-FCW w/ \textit{RISP}}, respectively.
The details of the t-FCW framework have been stated in the Section \ref{sec:framework}, while the framework is composed of four stages and the point clouds are downsampled by a ratio of 2 at each stage (e.g., a point cloud with 1024 points is reduced to 512 points at the first stage). 
For the four-stage architecture, it is widely used in point cloud analysis \cite{ma2022rethinking,qian2022pointnext,zhang2023starting,zhu2024no}, ensuring fair comparison with other methods. 
For each downsampled point, its $K$ nearest neighbors are computed. 
After sampling, the PCSDs are required to be translated into t-FCW, and it is noticed that there exists a scale factor alpha \cite{tauzin2021giotto} for scaling the dimension relationship.

\subsection{Benchmarks}\label{sec:benchmarks}
\subsubsection{Classification on ModelNet40} \label{sec:modelnet40}
Table \ref{tab:permn40} shows the performance in point cloud classification with ModelNet40. We used the pioneering neural network PointNet series \cite{qi2017pointnet,qi2017pointnet++} as the flagship and used the pioneering non-parametric network Point-NN as the baseline. Meanwhile, it also includes Seg-NN~\cite{zhu2024no} and ICP-Classifier~\cite{gou20253d}, which aim to achieve interpretability and robustness without learning parameters. 
Compared to these methods, the t-FCW achieves the best trade-off between accuracy and efficiency. Specifically, with only \textit{xyz} coordinates as input, t-FCW obtains an overall accuracy of 84.8\%, outperforming all other non-parametric baselines in speed. While Point-GN~\cite{mohammadi2025pointgn} achieves slightly higher accuracy than ours, we conduct more experiments warranting deeper analysis to reveal that Point-GN's performance is highly sensitive to training configurations as shown in Fig.~\ref{fig:batch_shuffle}.

\begin{table}[ht]
\caption{Performance on ModelNet40}\label{tab:permn40}
\centering
\resizebox{\linewidth}{!}{
\begin{tabular}{l|c|c|c} 
\hline
Methods             & Points & Overall acc. (\%)                            & Speed                   \\ 
\hline
PointNet \cite{qi2017pointnet}            & 1024     & 89.2                                     & \multirow{2}{*}{223 \cite{ma2022rethinking}}   \\
PointNet++ \cite{qi2017pointnet++}          & 1024     & 90.7                                     &                         \\
PointMLP \cite{ma2022rethinking} & 1024 & 92.7  &  218\\
\hline
Point-NN \cite{zhang2023starting}     & 1024     & 81.8                                     & 209                     \\
Seg-NN \cite{zhu2024no}             & 1024     & 84.2                                     & 306                     \\
Point-GN \cite{mohammadi2025pointgn}             & 1024     & 85.3                                     & 301                     \\
ICP-Classifier \cite{gou20253d}         & -        & \textbf{85.7} & -                       \\ 
t-FCW                & 1024     & 84.8                     & \uline{1727}                    \\
t-FCW w/ \textit{GeoPCSD} \cite{qiu2022geometric}       & 1024     & 82.9                                     & 660                     \\
\hline
PointMLP+t-FCW$^*$ & 1024 & 93.1 \textcolor{red}{($\uparrow 0.4$)} & 199 \textcolor{blue}{($\downarrow 8.7\%$)} \\
Point-NN+t-FCW$^*$ & 1024 & 85.3 \textcolor{red}{($\uparrow 3.5$)} & 163 \textcolor{blue}{($\downarrow 6.6\%$)}\\
\hline
\end{tabular}
}
\caption*{$^*$: We adopt Lai~\etal~\cite{lai2025point}'s fusion strategy to obtain results.}
\end{table}

Notably, t-FCW processes 1727 samples per second, which is more than 6 times faster than Point-NN and significantly faster than Seg-NN and Point-GN. These results highlight the highly computational efficiency of the t-FCW framework. Meanwhile, t-FCW can serve as a complementary feature to existing methods, offering promising performance gains at a low loss of inference speed.

\subsubsection{Classification on ScanObjectNN}\label{sec:expsonn}

Table \ref{tab:sonn} shows the performance on point cloud classification with ScanObjectNN, which includes three standard evaluation splits: OBJ-BG, OBJ-ONLY, and PB-T50-RS. 
In this benchmarking, we can understand the effectiveness of t-FCW because the simplicity of ModelNet40 makes it hard to distinguish the real performance of the models. In contrast, the ModelNet40 does not fully contain such properties compared to the ScanObjectNN: Noisy, Irregular shapes, and Arbitrary rotation, etc. 
Overall, the t-FCW still shows a small gap compared to the pioneering non-parametric networks. However, this trade-off is anticipated due to the properties of the t-FCW representation: by collapsing the point cloud into a global structure aligned along the feature dimension, local noise or outlier points may be integrated into the aggregated representation, potentially degrading performance. Despite the limitation, t-FCW still serves as a valuable complementary feature to existing methods.

\begin{table}[ht]
\centering
\caption{Overall Accuracy (\%) on ScanObjectNN}\label{tab:sonn}
\resizebox{\linewidth}{!}{
\begin{tabular}{l|ccc} 
\hline
Methods    & OBJ-BG & OBJ-ONLY & PB-T50-RS                                 \\ 
\hline
PointNet \cite{qi2017pointnet}  & 73.3   & 79.2     & 68.2  \\
PointNet++ \cite{qi2017pointnet++} & 82.3   & 84.3     & 77.9                                      \\
PointMLP$^{*\dagger}$ \cite{ma2022rethinking} & 88.6 & 86.4 & 81.5 \\
\hline
Point-NN \cite{zhang2023starting}  & 70.1   & 73.8     & 63.8                                      \\
Seg-NN \cite{zhu2024no}    & -      & -        & 64.4                                      \\
Point-GN \cite{mohammadi2025pointgn}  & \uline{85.2}   & \uline{86.0}     & \textbf{86.4}                                      \\ 
t-FCW (global) \cite{lai2025point} & 58.9 & - & 53.9 \\
Point-FCW \cite{lai2025point} & 61.5 & - & 57.6 \\
t-FCW      & 66.3   & 71.9     & 55.6                                      \\
t-FCW w/ \textit{GeoPCSD} \cite{qiu2022geometric}     & 69.0   & 72.0     & 57.8    \\
\hline
PointMLP+t-FCW$^{*\dagger}$ & \textbf{90.0} \textcolor{red}{($\uparrow 1.4$)} & \textbf{87.8} \textcolor{red}{($\uparrow 1.4$)} & \uline{82.9} \textcolor{red}{($\uparrow 1.4$)} \\
Point-NN+t-FCW$^{*\dagger}$ & 70.9 \textcolor{red}{($\uparrow 0.8$)} & 75.7 \textcolor{red}{($\uparrow 1.9$)} & 65.8 \textcolor{red}{($\uparrow 2.0$)}\\
\hline
\end{tabular}
}
\caption*{$^*$: We adopt Lai~\etal~\cite{lai2025point}'s fusion strategy to obtain results.\\
$^\dagger$: We used the pretrained model whose parameters were trained by PB-T50-RS.}
\end{table}

\begin{table*}[ht]
\centering
\caption{Few-shot Performance (\%) on S3DIS}\label{tab:semseg}
\begin{tabular}{l|ccc|ccc|ccc|ccc} 
\hline
\multirow{3}{*}{Methods} & \multicolumn{6}{c|}{Two-way}                                   & \multicolumn{6}{c}{Three-way}                                    \\ 
\cline{2-13}
                         & \multicolumn{3}{c|}{One-shot} & \multicolumn{3}{c|}{Five-shot} & \multicolumn{3}{c|}{One-shot} & \multicolumn{3}{c}{Five-shot}  \\
\cline{2-13}
                         & \textit{$S_0$} & \textit{$S_1$}& \textit{avg}                 & \textit{$S_0$} & \textit{$S_1$} & \textit{avg}                  & \textit{$S_0$} & \textit{$S_1$} & \textit{avg}                 & \textit{$S_0$} & \textit{$S_1$} & \textit{avg}                  \\ 
\hline
DGCNN \cite{wang2019dynamic}                   & 36.34 & 38.79 & 37.57         & 56.49 & 56.99 & 56.74          & 30.05 & 32.19 & 31.12         & 46.88 & 47.57 & 47.23          \\
ProtoNet \cite{satorras2018few}                & 48.39 & 49.98 & 49.19         & 57.34 & 63.22 & 60.28          & 40.81 & \textbf{45.07} & 42.94         & 49.05 & 53.42 & 51.24          \\ 
\hline
Point-NN \cite{zhang2023starting}                & 42.12 & 42.62 & 42.37         & 51.91 & 49.35 & 50.63          & 38.00 & 36.21 & 37.10         & 45.91 & 43.44 & 44.67          \\
Seg-NN \cite{zhu2024no}                  & 49.45 & 49.60 & 49.53         & 59.40 & 61.48 & 60.44          & 39.06 & 40.10 & 39.58         & 50.14 & 51.33 & 50.74          \\
t-FCW                    & 43.42 & 43.54 & 43.48         & 52.50 & 53.84 & 53.17          & 35.69 & 36.15 & 35.92         & 43.24 & 46.77 & 45.00          \\
t-FCW w/ \textit{GeoPCSD}            & 45.86 & 44.23 & 45.05         & 55.46 & 54.45 & 54.20          & 38.62 & 37.15 & 37.88         & 46.01 & 48.06 & 47.04          \\
Seg-NN + t-FCW & \textbf{52.66} & \textbf{54.26} & \textbf{53.46}         & \uline{61.85} & \uline{63.39} & \uline{62.62}          & \uline{42.32} & \uline{43.73} & \uline{43.03}         & \uline{51.94} & \textbf{54.48} & \uline{53.21}          \\
Seg-NN + t-FCW w/ \textit{GeoPCSD} & \uline{52.55} & \uline{53.49} & \uline{53.02}         & \textbf{62.42} & \textbf{63.53} & \textbf{62.98}          & \textbf{42.70} & 43.68 & \textbf{43.19}         & \textbf{52.15} & \uline{54.46} &  \textbf{53.31}         \\
\hline
\end{tabular}
\end{table*}

Our empowered t-FCW outperforms the naive t-FCW baseline from Lai~\etal~\cite{lai2025point}, and the t-FCW w/ \textit{GeoPCSD} can reach higher performance than that of Point-FCW \cite{lai2025point} without incorporating hybrid topological features via TDA.
Although t-FCW does not achieve the highest accuracy on the ScanObjectNN benchmark, it serves a distinct and valuable role by offering a unified and interpretable framework for evaluating PCSDs. Unlike previous works such as~\cite{ran2022surface,qiu2022geometric}, which propose new PCSDs in conjunction with specialized neural network architectures, our framework decouples the descriptor from the network design. This decoupling enables a straightforward comparison to PCSDs that isolate the effect of the neural networks.
For instance, when \textit{GeoPCSD} \cite{qiu2022geometric} are integrated into the t-FCW framework, they modestly improve performance across all three ScanObjectNN splits (Table~\ref{tab:sonn}). Specifically, the Geometric PCSD helps to increase accuracy on the OBJ-BG split from 66.3\% to 69.0\% and on PB-T50-RS from 55.6\% to 57.8\%. Based on a consistent platform to be fairly evaluated, combining these improvements and the research of Qiu~\etal \cite{qiu2022geometric} suggests that \textit{GeoPCSD} indeed contributes meaningful priors, particularly under noise.
Therefore, rather than competing purely on absolute performance, the t-FCW framework provides a reliable and PCSDs-architecture-independent platform to examine the utility and generalizability of various PCSDs across tasks and datasets.

\subsubsection{Semantic Segmentation}
We evaluate the t-FCW framework for few-shot semantic segmentation \cite{zhao2021few,PAPFZS3D,zhu2024no} on the S3DIS, which segments novel semantic categories from only a few labeled examples per class.
Table \ref{tab:semseg} presents the results under two-way and three-way segmentation with both one-shot and five-shot settings. We compare our models (t-FCW and t-FCW w/ \textit{GeoPCSD}) against state-of-the-art parametric and non-parametric baselines, including DGCNN, ProtoNet, Point-NN, and Seg-NN.
Overall, Seg-NN achieves the highest mIoU, and t-FCW achieves competitive results across all splits, outperforming Point-NN in nearly every setting and approaching the performance of Seg-NN. 
Moreover, we observe that t-FCW w/ \textit{GeoPCSD} consistently improves segmentation performance in both one-shot and five-shot settings. 
These results also demonstrate that the geometric attributes from \textit{GeoPCSD} provide complementary structural information that enhances the segmentation capability of t-FCW, and continue to support the discussion in the Section \ref{sec:expsonn}.
Beyond the standalone performance of t-FCW, we further investigate its role as a complementary module by integrating it into the Seg-NN framework. As reported in the Table \ref{tab:semseg}, Seg-NN+t-FCW and Seg-NN+t-FCW w/ \textit{GeoPCSD} both deliver consistent improvements over the vanilla Seg-NN across all few-shot settings. These improvements confirm that t-FCW provides further complementary benefits. The results from the Table \ref{tab:permn40}, Table \ref{tab:sonn}, and Table \ref{tab:semseg} strongly support the fact that t-FCW not only serves as an interpretable non-parametric representation but also functions as a plug-in module to enhance existing state-of-the-art networks.

\subsubsection{Parts Segmentation}
We follow the standard split of ShapeNet-Part, which includes 16 object categories and 50 annotated part labels, and evaluate performance using the mean Intersection over Union (mIoU) per shape. As shown in Table \ref{tab:partseg}, the t-FCW framework achieves the same performance (70.4\%), matching the overall mIoU of Point-NN (70.4\%) while offering significantly higher inference speed (80 samples/second) with one sample per batch. 
Furthermore, while processing 128 samples per batch, the speed of t-FCW improves dramatically to 333 samples/second, and its performance is better than Point-NN (70.5\%). 
Even with the model of t-FCW w/ \textit{GeoPCSD}, it still maintains a competitive speed of 178 samples/second.
We notice that the \textit{GeoPCSD} does not improve part segmentation performance, but it does improve semantic segmentation performance. We present an empirical study for this in the section. \ref{sec:cor}.  

\begin{table}[ht]
\centering
\caption{Overall Performance (\%) on ShapeNet-Part}\label{tab:partseg}
\begin{tabular}{l|c|c|c} 
\hline
Methods & Batch size & Overall mIoU & Speed                  \\ 
\hline
PointNet \cite{qi2017pointnet}   & -             & 83.7    & - \\
PointNet++ \cite{qi2017pointnet++}   & -             & 85.1  & 45   \\ 
\hline
Point-NN$^*$ \cite{zhang2023starting}   & 1             & \uline{70.4}  &55   \\
t-FCW$^*$          &1          & \uline{70.4}  &80 \\
t-FCW             & 128       & \textbf{70.5}      & \textbf{333}     \\
t-FCW w/ \textit{GeoPCSD}    & 128    & 66.8  & \uline{178} \\ 
\hline
\end{tabular}
\end{table}

\begin{table}[ht]
\centering
\caption{Performance on ModelNet40 w/ Rotations}\label{tab:rotation}
\resizebox{\linewidth}{!}{
\begin{tabular}{l|c|ccc} 
\hline
\multirow{2}{*}{Methods} & \multirow{2}{*}{Baseline$^*$} & \multicolumn{3}{c}{Rotation}  \\ 
\cline{3-5}
                         &                           & $z/z$  & $z/SO(3)$ & $SO(3)/SO(3)$  \\ 
\hline
PointNet \cite{qi2017pointnet}                 & 89.2                      & 89.2 & 16.4    & 75.5         \\
PointNet++ \cite{qi2017pointnet++}              & 90.7                      & 91.8 & 18.4    & 77.4         \\
DGCNN \cite{wang2019dynamic}                    & 92.9                      & 92.2 & 20.6    & 81.1         \\ 
\hline
Point-NN \cite{zhang2023starting}                & 81.8                      & 69.4 & 19.9    & 39.4         \\
Point-GN \cite{mohammadi2025pointgn}                & 85.3                      & 71.9 & 20.5    & 39.4         \\
ICP-Classifier \cite{gou20253d}          & 85.7                      & -    & 73.7    & -            \\ 
\hline
t-FCW                    & 84.8                      & 69.5 & 20.4    & 38.2         \\
t-FCW w/ \textit{GeoPCSD} \cite{qiu2022geometric}      & 82.9                      & 71.0 & 25.9    & 43.2         \\
t-FCW w/ \textit{RISP} \cite{zhang2024risurconv}       & 73.9                      & \textbf{73.9} & \textbf{73.9}    & \textbf{73.9}         \\
\hline
\end{tabular}
}
\caption*{$^*$: We define the baseline performance with no rotation applied to the ModelNet40.}
\end{table}

\subsection{Robustness}
\subsubsection{Rotation} \label{sec:rotation}

We report the robustness performance using ModelNet40 across three rotation scenarios, i.e., $z/z$ is trained and tested with z-axis rotations, $z/SO(3)$ is trained with z-rotations and tested with $SO(3)$ rotations, and $SO(3)/SO(3)$ is trained and tested under $SO(3)$ rotation. As shown in Table \ref{tab:rotation}, the pioneering deep learning networks and non-parametric methods demonstrate significant performance degradation under three rotation scenarios, which means poor robustness under rotation. Meanwhile, the baseline of t-FCW w/ \textit{RISP} is comparatively lower than other baselines, achieving only 73.9\% overall accuracy. However, this approach suffers no performance degradation, consistently achieving 73.9\% accuracy across all rotation settings.
As established in Theorem \ref{theorem:rvgm}, the standard Gram matrix computed from 3D Cartesian coordinates is not invariant under rotation. To address this problem, we can introduce a PCSD (as established in Theorem \ref{theorem:ritfcwwrisp}) that enforces rotational invariance before computing the Gram matrix to improve the robustness under rotation.

\subsubsection{Effect of Batch Size and Data Shuffling}\label{sec:ebsds}

While Table~\ref{tab:permn40} shows Point-GN achieving higher accuracy than t-FCW on ModelNet40 (85.3\% vs 84.8\%) and much higher accuracy than t-FCW on ScanObjectNN (86.4\% vs 57.8\%), reported accuracies from single configurations may not reflect model reliability. Feature privacy is crucial for non-parametric frameworks using memory banks, to prevent leaking feature information that could compromise fairness or cause overfitting via implicit memorization. We evaluate stability across batch sizes and data ordering to provide a comprehensive comparison.
We evaluate t-FCW, Point-NN, and Point-GN using a shared non-parametric classifier (Equation (\ref{eq:tfeq})) across batch sizes (1, 2, 4, 8, 16, 32) with and without data shuffling. Data shuffling is a standard practice in deep learning to prevent overfitting and ensure generalization. We test on both ScanObjectNN (PB-T50-RS) and ModelNet40. 

Fig.~\ref{fig:batch_shuffle} compares accuracy across models under varying batch sizes and shuffling. t-FCW shows high stability, with consistent accuracy minimally affected by shuffling. For example, Point-NN improves with larger batches without shuffling (up to 63.39\% at 32), but degrades sharply with shuffling, indicating reliance on batch correlations. 
Point-GN exhibits high reliability on model configuration dependence. It peaks at 86.49\% at batch 16 without shuffling but collapses with shuffling, revealing dependence on sample order; it also oscillates at small batches (1-4) on the ScanObjectNN. On ModelNet40, the instability is even more extreme: accuracy ranges from 3.36\% to 85.53\% across configurations. However, it stops the disaster failures when the normalization is applied, which indicates a critical dependence on batch-wise information.

These results present new insights into the comparison in Table~\ref{tab:permn40}, which shows that Point-GN's 85.3\% accuracy on ModelNet40 represents peak performance under optimal conditions (batch=32, no shuffle). In contrast, t-FCW's 84.8\% accuracy represents consistent, reliable performance across all configurations. Our analysis reveals that achieving true "non-parametric" methods requires not only avoiding learnable parameters but also avoiding implicit dependencies on batch-wise information (i.e., feature privacy).

\begin{figure*}[ht]
    \centering
    \includegraphics[width=0.78\linewidth]{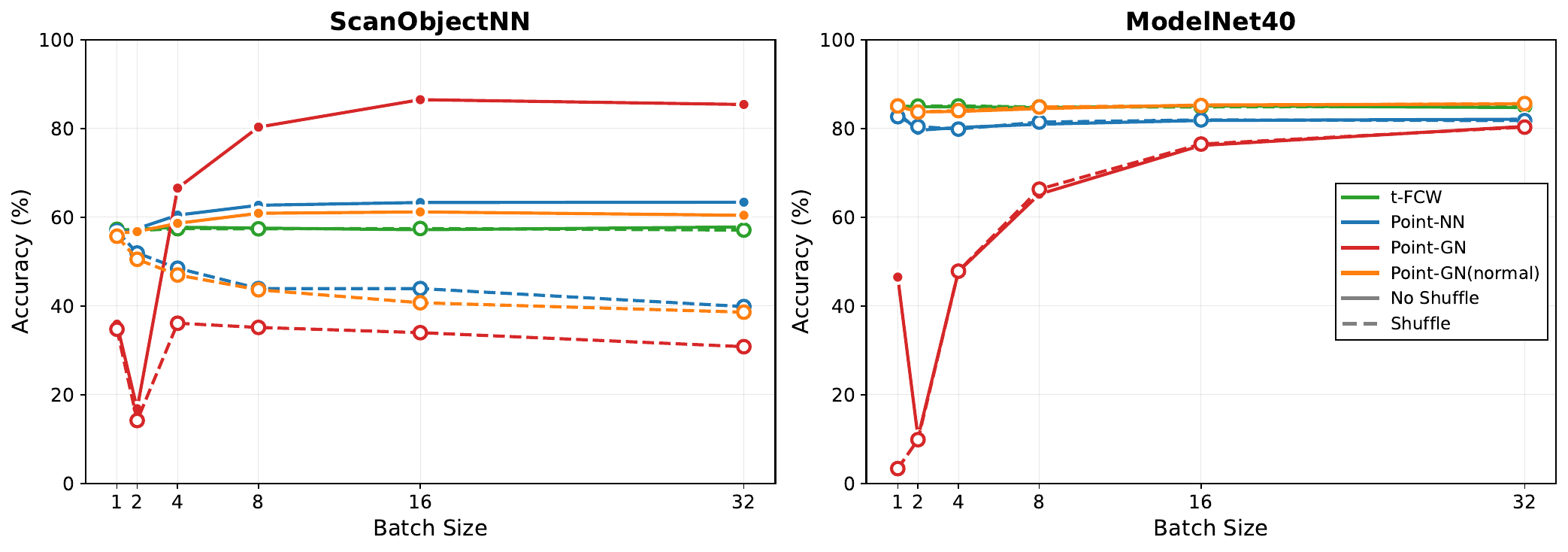}
    \caption{Stability analysis across batch sizes and data shuffling.}
    \label{fig:batch_shuffle}
\end{figure*}

\subsubsection{3D Adversarial Attack}
3D Adversarial Attacks aim to mislead models by introducing subtle perturbations to the 3D point cloud, resulting in incorrect predictions or decisions. To assess model robustness against such attacks, we conduct evaluations on the ModelNet40 dataset using the adversarial benchmark proposed by Xiang~\etal~\cite{xiang2019generating}, which includes adversarial examples for the 10 largest object classes in ModelNet40. The evaluation procedure is as follows: each model first performs inference on the original (clean) objects; we then identify and retain only those samples that are correctly classified. These correctly predicted samples are subsequently transformed into adversarial samples and used to attack the models. The robustness is quantified by the number of attack success samples, as presented in Table~\ref{tab:adv_attack}, whose attack types also include Adversarial Clusters, Adversarial Objects, and Adversarial Point Perturbations based on the definition which is proposed by Xiang~\etal~\cite {xiang2019generating}.

\begin{table}[ht]
\centering
\caption{Adversarial Attack on ModelNet40}\label{tab:adv_attack}
\resizebox{\linewidth}{!}{
\begin{tabular}{l|c|ccc} 
\hline
\multirow{2}{*}{Methods} & \multirow{2}{*}{Samples} & \multicolumn{3}{c}{Attack Success Samples}  \\ 
\cline{3-5}
                         &                          & Clusters & Objects & Perturbations          \\ 
\hline
PointMLP \cite{ma2022rethinking}                & \textbf{666}                      & 5        & 23      & 221                    \\
Point-NN  \cite{zhang2023starting}                & 492                      & 5        & 15      & 38                     \\
Point-GN \cite{mohammadi2025pointgn}                 & 440                      & \textbf{4}        & 24      & 49                     \\ 
\hline
t-FCW                    & 486                      & 5        & 19      & \textbf{22}                     \\
t-FCW w/ Geo PCSD \cite{qiu2022geometric}      & 549                      & \textbf{4}        & \textbf{13}      & 36                     \\
\hline
\end{tabular}
}
\end{table}

Overall, most of the listed models are able to defend against the attack by Adversarial Clusters and Adversarial Objects. However, clear differences in robustness emerge when considering Adversarial Point Perturbations. PointMLP shows the poorest performance, which is fooled by 221 perturbation samples successfully, even though it achieves the highest number of correct predictions on clean samples. We assume that PointMLP is the only model optimized by the gradient. 
In contrast, other non-parametric models demonstrate strong robustness against perturbations. Among all evaluated models, t-FCW and t-FCW w/ \textit{GeoPCSD} exhibit the highest robustness, with only 22 and 36 successful adversarial perturbation attacks from 486 and 549 adversarial samples, respectively, making them the most robust models in this category. Similarly, Point-NN and Point-GN also perform better than PointMLP in Adversarial Point Perturbations, though they lag behind t-FCW and t-FCW w/ \textit{GeoPCSD}.
These results highlight the effectiveness of the t-FCW framework design in robustness to adversarial point-level attacks, making it a promising candidate for applications requiring both interpretability and stability from adversarial attacks.

\subsubsection{Corruption}\label{sec:cor}

To diagnose the performance gap of t-FCW on ScanObjectNN and ground the GeoPCSD drop in part segmentation, we evaluate two corruptions~\cite{ren2022modelnetc}: jitter and global noise with severity levels (S1--S5) on ModelNet40 and ShapeNet-Part. 
As shown in Table~\ref{tab:corruption_model}, Algorithms \ref{alg:tfcw} and \ref{alg:etfcw} (for global and local features, respectively) are robust to jitter noise, where Algorithm \ref{alg:etfcw} performs better. However, under global noise, the pattern reverses: Algorithm \ref{alg:tfcw} outperforms Algorithm \ref{alg:etfcw}. This exposes a vulnerability of local neighborhood aggregation that, when scene-level outliers contaminate $k$-NN groupings, they disrupt the neighborhood selection process and degrade the discriminative inter-dimensional correlations. This directly mirrors the challenge posed by PB-T50-RS of ScanObjectNN, where background clutter acts as globally distributed outliers to degrade t-FCW performance. Therefore, we utilize the United block as shown in Fig.~\ref{fig:tfcw_block}, which offers the best trade-off by combining Algorithms \ref{alg:tfcw} and \ref{alg:etfcw}.
Furthermore, Tables~\ref{tab:corruption_model} and~\ref{tab:corruption_shape} together demonstrate why GeoPCSD is more robust to xyz under corruption, although Table~\ref{tab:partseg} shows that t-FCW performs better than t-FCW w/ GeoPCSD on clean data. This results in a precision-robustness trade-off: the vector-based information in GeoPCSD provides stable representations under noise by maintaining geometric relationships even under perturbations, while this information introduces redundancy on clean data. In contrast, xyz captures a compact representation with lower robustness, but generally performs better on clean datasets (e.g., ModelNet40 and ShapeNet-Part).

\begin{table}[!t]
  \renewcommand{\arraystretch}{1.1}
  \caption{ModelNet40 w/ Corruption (Acc. \%)}
  \label{tab:corruption_model}
  \centering
  \scriptsize
  \setlength{\tabcolsep}{3.2pt}
  \begin{tabular}{c|c|c|cccccc}
    \toprule
    \textbf{PCSD} & \textbf{t-FCW} & \textbf{Corruption} & \textbf{S1} & \textbf{S2} & \textbf{S3} & \textbf{S4} & \textbf{S5} & \textbf{Drop$\downarrow$} \\
    \midrule
    \multirow{6}{*}{\textit{GeoPCSD}}
      & Alg. \ref{alg:tfcw}  &     & 78.24 & 72.41 & 62.84 & 50.61 & 42.06 & 36.75 \\
      & Alg. \ref{alg:etfcw}  & Jitter   & 80.23 & 79.13 & 75.97 & 68.03 & 59.76 & \textbf{21.64} \\
      & Alg. \ref{alg:tfcw} \& \ref{alg:etfcw} &     & 82.78 & 80.79 & 76.42 & 67.26 & 59.56 & 23.54 \\
      \cline{2-9}
      & Alg. \ref{alg:tfcw}  &    & 73.38 & 66.25 & 60.41 & 56.00 & 52.03 & \textbf{26.78} \\
      & Alg. \ref{alg:etfcw}   & Global  & 54.58 & 40.96 & 31.20 & 26.38 & 22.08 & 59.32 \\
      & Alg. \ref{alg:tfcw} \& \ref{alg:etfcw}  &    & 68.68 & 57.01 & 46.84 & 38.82 & 33.02 & 50.08 \\
    \midrule
    \multirow{6}{*}{\textit{xyz}}
      &  Alg. \ref{alg:tfcw}  &    & 80.27 & 75.41 & 64.51 & 53.93 & 46.07 & 35.33 \\
      & Alg. \ref{alg:etfcw}   & Jitter   & 82.70 & 82.05 & 78.48 & 73.14 & 66.65 & \textbf{16.73} \\
      & Alg. \ref{alg:tfcw} \& \ref{alg:etfcw}  &    & 84.40 & 82.54 & 78.28 & 69.69 & 60.90 & 23.82 \\
      \cline{2-9}
      &  Alg. \ref{alg:tfcw}  &   & 73.50 & 66.17 & 58.51 & 51.26 & 47.16 & \textbf{34.24} \\
      & Alg. \ref{alg:etfcw}   & Global  & 44.69 & 33.43 & 23.58 & 20.14 & 17.42 & 65.96 \\
      & Alg. \ref{alg:tfcw} \& \ref{alg:etfcw}  &   & 58.02 & 45.18 & 36.51 & 29.70 & 24.84 & 59.89 \\
    \midrule
  \end{tabular}
\end{table}

\begin{table}[!t]
  \caption{ShapeNet-Part w/ Corruption (mIoU \%)}
  \label{tab:corruption_shape}
  \centering
  \scriptsize
  \setlength{\tabcolsep}{5pt}
  \begin{tabular}{c|c|cccccc}
    \toprule
    \textbf{PCSD}  & \textbf{Corruption} & \textbf{S1} & \textbf{S2} & \textbf{S3} & \textbf{S4} & \textbf{S5} & \textbf{Drop$\downarrow$} \\
    \midrule
    \multirow{2}{*}{\textit{GeoPCSD}}
       & Jitter   & 66.91 & 65.97 & 64.53 & 62.90 & 60.58 & \textbf{7.62}  \\
       & Global  & 66.24 & 65.38 & 64.76 & 64.05 & 63.50 & \textbf{4.70}  \\
    \midrule
    \multirow{2}{*}{\textit{xyz}}
       & Jitter   & 62.49 & 61.74 & 60.83 & 59.34 & 57.57 & 12.62 \\
      & Global  & 61.63 & 60.67 & 59.74 & 59.02 & 58.20 & 11.98 \\
    \bottomrule
    \hline
  \end{tabular}
\end{table}

\subsection{Ablation}\label{sec:ablation}
\subsubsection{Diagonal Elements}
As stated in the Section \ref{sec:interpretability}, it is noticed that the formulation of the t-FCW can be represented by a Gram matrix.
To further investigate the impact of the diagonal elements in the t‑FCW formulation, we conduct an ablation study by modifying or removing the components of the formulation. Table~\ref{tab:ablation_diag} reports the overall classification accuracy on ModelNet40 for point cloud classification. The original t‑FCW formulation in Equation (\ref{eq:tfcwG}), which excludes diagonal (self-correlation) elements, achieves the highest accuracy (84.8\%). In contrast, when diagonal components are retained, either explicitly or through inclusion of the full Gram matrix, the performance consistently degrades.
Critically, simply removing diagonals via $G - \text{diag}(G)$ yields only 81.3\% accuracy, which is identical to using diagonal elements alone and 3.5\% lower than t-FCW.
Additionally, using the Gram matrix directly also results in suboptimal performance compared to t-FCW, reinforcing the hypothesis that the diagonal elements contribute negatively. These results validate the design of excluding self-correlations in t‑FCW to enhance discriminability.

\begin{table}[ht]
\caption{Ablation Study on Diagonal elements}\label{tab:ablation_diag}
\centering
\begin{tabular}{l|c} 
\hline
Formulation & Overall Acc. (\%) \\ 
\hline
$\sqrt{diag(G)1^{T}_{n}+1_{n}diag(G)^{T}-2G}$       &    \textbf{84.8}      \\
$\sqrt{diag(G)1^{T}_{n}+1_{n}diag(G)^{T}-1G}$       &    82.5               \\
$\sqrt{diag(G)1^{T}_{n}+1_{n}diag(G)^{T}}$       &       81.3               \\
$G$        &     82.6 \\
$G-diag(G)$       &       81.3               \\
\hline
\end{tabular}
\end{table}

\subsubsection{Normalization}
Point clouds are naturally irregular, unevenly scattered, and differently dense, and formatted with absolute coordinates that are shifted under various scenarios. Standard sampling and grouping strategies (i.e., FPS and $k$-NN) preserve these intrinsic properties, so some existing models \cite{ma2022rethinking,zhang2023starting,mohammadi2025pointgn} employ normalization to regularize local geometric distributions. Following this principle, we introduce normalization into the Algorithm \ref{alg:etfcw}, and we report that there exists a performance loss without applying normalization, as shown in Table \ref{tab:ion}.

\begin{table}[ht]
\caption{Normalization for ModelNet40 Accuracy (\%)}\label{tab:ion}
\centering
\begin{tabular}{l|l|l} 
\hline
Methods  & w/ normalization & w/o normalization  \\ 
\hline
PointMLP \cite{ma2022rethinking} & 92.9             & 33.6 \textcolor{blue}{($\downarrow 59.3\%$)}              \\
Point-NN \cite{zhang2023starting} & 81.8             & 76.2 \textcolor{blue}{($\downarrow 5.6\%$)}               \\
Point-GN \cite{mohammadi2025pointgn} & 85.3             & 76.1 \textcolor{blue}{($\downarrow 9.2\%$)}               \\ 
t-FCW (our)    & 84.8             & 80.4 \textcolor{blue}{($\downarrow 4.0\%$)}       \\
\hline
\end{tabular}
\end{table}

\subsubsection{Sensitivity to the K value}
As discussed in Section~\ref{sec:sampling}, several studies~\cite{zhang2023starting,zhu2024no,lai2025point} assume that local neighborhoods provide reliable geometric information.
Our framework, which also follows the assumption, is sensitive to the $K$-nearest neighbors for performance, so we empirically evaluate the impact given by $K$. We normalize performance to highlight minor variations for a facilitated comparison. As illustrated in Fig.~\ref{fig:np}, we recommend selecting $K$-nearest neighbors across different datasets by highlighting star marks.

\begin{figure}[ht]
    \centering
    \includegraphics[width=0.85\linewidth,trim=0 10 0 0, clip]{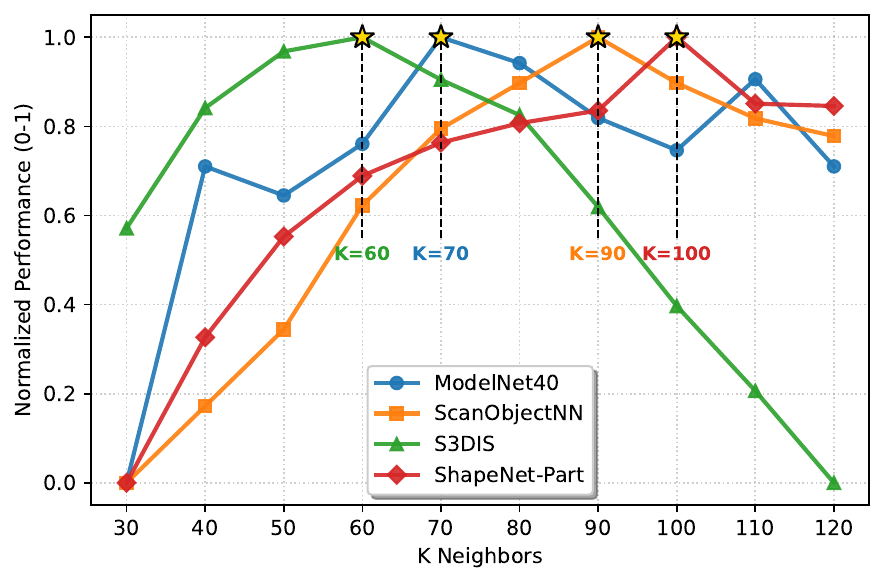}
    \caption{Normalized vs. $K$-Nearest Neighbors}
    \label{fig:np}
\end{figure}

\subsection{Complexity Analysis}
\subsubsection{Computational Complexity}
We analyze the theoretical complexity of t-FCW compared to existing non-parametric feature extractors to demonstrate its scalability advantages. Given a point cloud $P \in \mathbb{R}^{N \times 3}$, the existing methods express below computational complexity:
\begin{itemize}
    \item \textbf{t-FCW \cite{lai2025point}: }
    It translates the PCSD ($\mathcal{S} \in \mathbb{R}^{N \times C}$) into a compact size representation ($\mathbb{R}^{C \times C}$) with complexity $\mathcal{O}(C^2N)$ where $N \gg C$. Its output shape is $C^2$.
    \item \textbf{Empowered t-FCW (Ours): }
    It translates the PCSD ($\mathcal{S} \in \mathbb{R}^{N \times K \times C}$) with $K$ nearest neighbors by a complexity $\mathcal{O}(C^2NK)$, similarly where $N \gg K \gg C$. Its output shape is $N \times C^2$.
    \item \textbf{Point-NN \cite{zhang2023starting}: }It processes the embedded point cloud ($F \in \mathbb{R}^{N \times K \times D_1}$) with $K$ nearest neighbors using trigonometric functions with complexity $\mathcal{O}(D_1NK)$. Its output shape is $N \times D_1$.
    \item \textbf{PointMLP \cite{ma2022rethinking}: }It processes the embedded point cloud $F \in \mathbb{R}^{N \times K \times D_2}$ with $K$ nearest neighbors using MLP functions with complexity $\mathcal{O}(D_2^2NK)$. Its output shape is $N \times D_2$.
\end{itemize}
Since $D \gg C^2$ typically (i.e., $C\in[6, 14]$, $D_1\in [72,1152]$, $D_2\in [64,1024]$), we have $C^2NK < DNK \ll D^2NK$. Our method balances efficiency and expressiveness, introducing $K$ nearest neighbors overhead over t-FCW \cite{lai2025point} while remaining more efficient than Point-NN and PointMLP.

\subsubsection{Volume Scalability}
We evaluate various frameworks on large-volume point clouds by progressively increasing the number of points until a "CUDA out of memory" (OOM) error occurs. 
A random point cloud containing 1024 points is initially generated and fed into the model. If no error occurs (i.e., CUDA out of memory), the volume is enlarged by 1024 points in the next step. This process continues iteratively until the GPU memory is exhausted.

We used an RTX A5000 with 24GB to carry out this experiment. As shown in Fig. \ref{fig:cmuudsv}, the t-FCW demonstrates superior efficiency, processing up to 77,824 points before exhaustion. In contrast, existing frameworks reach their memory limits much earlier: Point-NN can process volume at 38,912 points, PointMLP at 35,840 points, and PointMLP (elite) at 56,320 points, and they fail at the next step (highlighted with an error break). Fig. \ref{fig:cmuudsv} shows that the t-FCW grows significantly more slowly than the competitive methods \cite{zhang2023starting,ma2022rethinking} with the same memory limitation that supports its robustness for large-scale point cloud processing.

\begin{figure}[ht]
    \centering
    \includegraphics[width=\linewidth,trim=0 20 0 0, clip]{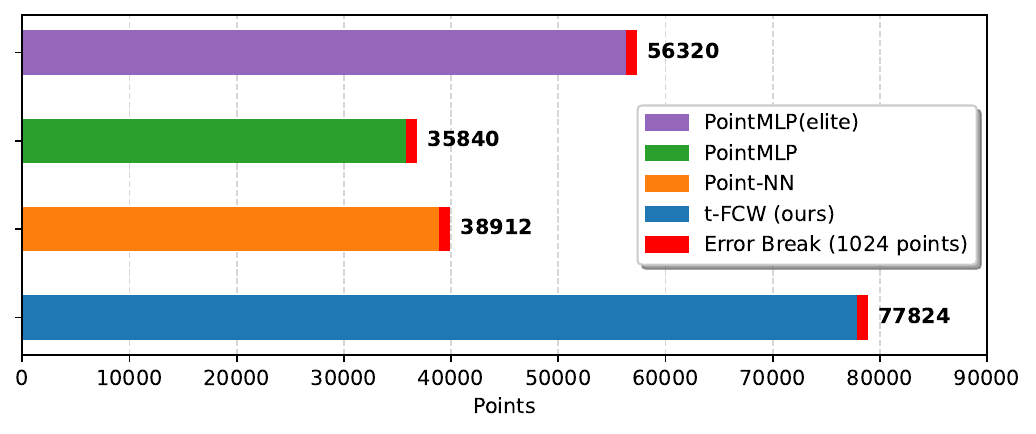}
    \caption{Max processing points' volume of various methods under the same memory limitation.}
    \label{fig:cmuudsv}
\end{figure}

\section{Conclusion}\label{sec:conclu}
We have presented a non-parametric and interpretable t-FCW framework for point cloud analysis. Extending it to segmentation and integrating diverse PCSDs yields strong efficiency, robustness to rotations, randomness, and 3D adversarial attacks, while preserving interpretability in classification and segmentation across four benchmarks. 
Our experiments reveal an important design principle that descriptor complexity should match data characteristics. The t-FCW performs well on clean datasets, while t-FCW w/ \textit{GeoPCSD} improves performance on noisy, real-world data by leveraging extra geometric descriptors.
Interpretability is demonstrated through a mathematical formulation revealing strong inter-dimensional correlations in shapes, and visualizations showing that similar semantic regions share t-FCW patterns, unlike dissimilar ones. Overall, t-FCW offers a compact and expressive representation, serving as a standalone solution or a complementary module for 3D point cloud processing.
In future work, we plan to utilize the empowered t-FCW graph representation to real-world scenarios, where it may serve as a complementary module within SLAM systems.

\bibliographystyle{IEEEtran}
\bibliography{ref.bib}

\end{document}